\documentclass[wcp]{jmlr}

\usepackage{longtable}% for long tables

\usepackage{booktabs}
\usepackage{stmaryrd}
\makeatletter
\let\Ginclude@graphics\@org@Ginclude@graphics 
\makeatother

\jmlrvolume{}
\jmlryear{2021}
\jmlrworkshop{Under Review}

\mathchardef\hyph="2D
\newcommand{\sA}{\mathcal{A}}
\newcommand{\sB}{\mathcal{B}}
\newcommand{\sC}{\mathcal{C}}

\newcommand{\sP}{\mathcal{P}}

\newcommand{\sS}{\mathcal{S}}

\DeclareMathOperator*{\argmax}{argmax}

\DeclareMathOperator*{\simf}{sim}
\newcommand{\util}{\mathrm{util}}

\title[Max-Utility Based Arm Selection Strategy]{Max-Utility Based Arm Selection Strategy For Sequential Query Recommendations}

  \author{
   \Name{Shameem A. Puthiya Parambath} \Email{sham.puthiya@glasgow.ac.uk}\\
   \Name{Christos Anagnostopoulos} \Email{christos.anagnostopoulos@glasgow.ac.uk}\\
   \Name{Roderick Murray-Smith} \Email{roderick.murray-smith@glasgow.ac.uk}\\
   \Name{Sean MacAvaney} \Email{sean.macavaney@glasgow.ac.uk}\\
   \addr University of Glasgow, Glasgow, UK
\AND  % Authors with different addresses:
  \Name{Evangelos Zervas} \Email{ezervas@uniwa.gr}\\
  \addr University of West Attica, Greece
 }

\begin{document}

\maketitle

\begin{abstract}
We consider the query recommendation problem in closed loop interactive learning settings like online information gathering and exploratory analytics.
The problem can be naturally modelled using the Multi-Armed Bandits (MAB) framework with countably many arms.
The standard MAB algorithms for countably many arms begin with selecting a random set of candidate arms and then applying standard MAB algorithms, e.g., UCB, on this candidate set downstream.
We show that such a selection strategy often results in higher cumulative regret and to this end, we propose a selection strategy based on the maximum utility of the arms.
We show that in tasks like online information gathering, where sequential query recommendations are employed, the sequences of queries are correlated and the number of potentially optimal queries can be reduced to a manageable size by selecting queries with maximum utility with respect to the currently executing query.
Our experimental results using a recent real online literature discovery service log file demonstrate that the proposed arm selection strategy improves the cumulative regret substantially with respect to the state-of-the-art baseline algorithms. % and commonly used random selection strategy for a variety of contextual multi-armed bandit algorithms.
Our data model and source code are available at ~\url{https://anonymous.4open.science/r/0e5ad6b7-ac02-4577-9212-c9d505d3dbdb/}.
\end{abstract}

\section{Introduction}
We consider the problem of query recommendation in closed loop interactive computing environments, like exploratory data analysis, and online information discovery/gathering.
In such applications, a user (data analyst) starts off a session by issuing an initial query related to a \emph{topic of investigation/exploration} to the data system and, then, exploring the topic in-depth by executing further queries.
Query recommendation algorithms are employed to recommend these `future' further queries based on previously issued queries to improve user experience \citep{baeza2004query}.
Furthermore, in this context, query recommendation algorithms can be envisaged as a pillar component in resource-efficient decision making methods in data management systems.
In addition to providing recommendations for data exploration tasks, effective query recommendation algorithms are expected to greatly improve the way current data systems work (including query processing, pre-fetching data, pre-analysing data, managing cached data).

Forecasting the correct next-query helps the system to proactively reserve resources and be prepared for any immediate downstream tasks.
For example, if the forecast next query requires loading (or transferring through the network) a considerably huge amount of data to the main memory or running a CPU intensive task, the system can prefetch the data, free up caches and memory/processes to accommodate the tasks to carry out in the immediate short-term future.
These advantages of timely next-query forecasting will anticipate impact on the way data centers work and schedule tasks in terms of throughput and scalability in task scheduling.
Hence, an efficient query recommendation (forecasting) engine not only improves user engagement but also improves system performance and application-driven Quality of Experience.

Typical use case scenarios include: \emph{(i)} data scientists analyzing a large volume of data to obtain in-depth knowledge about the data for follow-up tasks like, %feature extraction,
data trend explanation \citep{savva2020}, report summarization \citep{marcel2011survey} and \emph{(ii)} users gathering information by discovering scholarly articles to conduct literature review using online services \citep{krause2011submodularity}.
These tasks can be conceptualized as the system 
recommending queries and the user either accepting or rejecting these recommendations, thus forming a closed loop interactive environment.
This learning environment can be viewed as a repeated game between an online algorithm and the user.
At each trial $t = 1,\ldots,T$, the algorithm chooses a query for recommendation from a very large query set $\sA$.
Based on the utility of the recommendation, the user chooses either to execute the query or ignore it, thus resulting in a reward.
Like in the standard recommendation settings, the rewards are in the form of `clicks' for the correctly recommended queries.

The MAB framework is popular in personalized recommender systems to model the trade-off between \emph{exploration} and \emph{exploitation} over a set of items (queries in our context) with unknown reward distribution. We model query recommendations using the MAB framework.
The standard MAB algorithms assume that the number of arms is fixed and relatively smaller than the number of trials.
In query recommendation, the total number of queries (arms) far exceeds the number of times users will use the recommendation platform.
Hence, the number of queries for recommendation is inherently huge and assumed to be countably many.
We do not consider the queries to be infinite as there can be many implausible queries and query sequences.
In addition, unlike in personalized recommendation settings, the queries can be about very general topics like COVID, MAB or SQL statements, and need not have a personalization component.
Our goal is to recommend the next query to be executed based on past executed queries and the \emph{topic of investigation}, i.e., currently executing query.

As opposed to the standard MAB settings, countably many armed bandit algorithms have to choose from an extremely large number of arms, often much larger than the number of experimental trials or time horizon.
The sheer volume of possible arms make it computationally impossible to try each of the arms even once.
The standard way to deal with countably many arms is to either randomly select a candidate set containing a fixed ($k$) but reduced number of arms ($k \ll T \ll |\sA|$) from the pool of arms, and run standard MAB algorithms on this reduced arm set \citep{wang2008algorithms,KalvitZ20,yinglun2020,bayati_2020}, or exploit the benign reward structure of the arms \citep{kleinberg_2019,magureanu2014lipschitz}.
As we demonstrate in our experimental evaluation, the random selection strategy often results in the optimal or near-optimal arms to be ignored in the selection process and, thus, hindering the performance of the bandit algorithm.
The algorithms that exploit the reward structure of the arm contexts also fall short of expectations.
As demonstrated in the experiments, the zooming based algorithm \citep{kleinberg_2019} performs sub-optimally and the OSLB \citep{magureanu2014lipschitz} is infeasible in our case as it requires one to solve an LP in each round.
Hence, the key component in the algorithm design for countably many armed bandit problems is the arm selection strategy.
To this end, we propose a strategy to select the candidate set of the most promising arms based on maximizing the utility of an arm with respect to the currently playing arm.

Given the currently executing query (currently playing arm)\footnote{At the start of the session, the initial query will be the currently playing arm and as user interacts with the system the currently playing arm will be the recommended query, if the user accepts the recommendation, a newly issued query, otherwise.}, we define the `goodness of arm for selection' with respect to the currently playing arm based on the notion of utility.
Our assumption is similar to the assumption that the rewards are Lipschitz function of the arm contexts \citep{magureanu2014lipschitz,kleinberg_2019}.
Furthermore, we model the likelihood, i.e., probability of a query to be preferred by the current arm as a function of pairwise similarity between them. 
Specifically, the currently executing query provides us with the context information of the topic the user is interested in, and we make use of the shared context information between different queries, represented as real valued vectors, to choose the best next query to run.
Then, we propose an effective selection strategy based on utility maximization. For theoretically plausible utility functions, our selection strategy reduces to monotonic submodular maximization problem with cardinality constraint.
For a very large set of arms, the submodular maximization problem with cardinality constraint can be solved in nearly linear time using a distributed greedy algorithm \citep{mirzasoleiman2016distributed}.
Our candidate set selection strategy allows us to dynamically set the number of arms in the candidate set, following the design principles \citep{berry1997bandit}: 
\emph{(i)} when the number of trials is large, the algorithm is allowed to sacrifice short term gain by eschewing arms with larger reward, if necessary, while exploring for arms with even larger reward that will expect to yield a long-term benefit, \emph{(ii)} when the number of trials is small, the algorithm is allowed to eschew new arm exploration in favor of selecting an arm that has a large immediate reward.
We also elaborate on the case of setting the value of~$k$ in `anytime' bandit settings.

The remainder of this paper is organized as follows: after a brief overview of the algorithms for countably many armed bandit problems and query recommendations in Section~\ref{sec:rel_work}, we describe our framework and elaborate on the candidate set selection algorithm in Section~\ref{sec:qr_rmd}.
In Section~\ref{sec:exp}, we report the results of our experimental study on real world sequential query recommendation for online literature discovery settings and compare against strong baselines like zooming algorithm \citep{kleinberg_2019}.
Section~\ref{sec:conc} concludes the paper.

\section{Related Work}
\label{sec:rel_work}
We present studies related to our problem in different areas.
%Before proceeding further, we have to make a clear distinction between the queries and the query representation.
%In our specific settings, we assume that the queries are represented using real valued vectors which represent the query contexts.
%As our context vectors are extracted using deep representation models which are based on vector space, our representations span corresponding vector space.
%But we don't need to emphasize this fact as we only require a similarity function defined over the query context vectors to do the selection.
%Moreover, we clearly don't assume that our query space is continuous but countably many only.
%Hence not every vector is a candidate set hence not continuous arms.
%\subsubsection{Multi Armed Bandits (MAB)}
\paragraph{\textbf{Multi-Armed Bandits:}}
We limit our discussion to Countably Many Armed Bandit (CMAB) algorithms.
For standard MAB algorithms, we recommend the readers to refer to \citet{cesa2006prediction,lattimore2020bandit}.
As opposed to standard MAB settings, in CMAB settings, bandit algorithms have to choose from an extremely large number of arms, often much larger than the possible number of experimental trials ($T$).
The sheer volume of possible arms makes it impossible to try each of the arms even once.
In stochastic bandit problems with countably many arms, the learner is restricted to ignore many arms without even trying them once, and dedicate the valuable exploration scheme only to a certain number of arms.
That is, in addition to the exploration-exploitation trade-off, which is typical to sequential learning algorithms, we need also to deal with the arm discovery-exploitation trade-off within the exploration phase. That is, while exploring, the algorithm has to decide whether it should try a new arm or revisit an already played arm to get a better estimate of the expected reward.
Recently, many algorithms have been proposed in the CMAB settings.
These algorithms can be broadly classified into: \emph{(i)} failure-based approaches and \emph{(ii)} pre-selection based approaches.

In failure-based approaches, there is no hard limit on the number of arms to explore. Such algorithms try different arms until the number of trials is exhausted or reach a predefined failure rate for the arms. Hence, the exploration phase lasts until the algorithms hit the end of time horizon.
In \citet{berry1997bandit}, the authors proposed \emph{k-failure}, where an arm is played until it incurs $k$ failures, and $m-run$, where \emph{1-failure} is used until $m$ arms are played or $m$ success is obtained, algorithms for Bernoulli arms.
Asymptotically, failure-based algorithms yield lower cumulative regret; but for finite $T < \infty$, they tend to perform poorly.
Additionally, failure-based algorithms require prior knowledge of $T$.
Similarly, \citet{herschkorn1996policies} proposed a \emph{non-recalling}, bounded memory, failure-based algorithm for Bernoulli arms. \citet{KalvitZ20} proposed an algorithm in a setting where the arms are partitioned into different types, and the goal is to find the arm from the superior type.
Their algorithm can be considered as failure-based, as it terminates only when the superior type is identified with large confidence.
The failure-based approaches are better suited for \emph{pure-exploration} settings. 

In the pre-selection based approach, only a specific number of arms are explored. This number can be either fixed in advance (before the start of the experiment) or adapted as the trial progresses.
Such approaches randomly choose $k$ arms from the pool of arms and use standard MAB algorithms on this subset of $k$ arms downstream.
Indicatively, the pre-selection based algorithm in \citet{wang2008algorithms} deals with selecting $k$ randomly chosen arms for exploration and exploitation.
The exact value of $k$ in \cite{wang2008algorithms} is defined as a function of the current trial count.
They also adopt the `goodness of arm' assumption which states that: the probability of the mean reward of a newly explored arm differing from the optimal arm is close to zero.
\citet{yinglun2020} proposed an algorithm in the CMAB settings with multiple best arms.
They also used $k$ arms selected uniformly at random from the arms pool.
In \citet{bayati_2020}, authors proposed a subsampling (uniformly at random without replacement) based greedy algorithm in CMAB settings.
The arm rewards of Lipschitz bandits \cite{magureanu2014lipschitz} are assumed to be \emph{smooth}, but arms are sampled sequentially till every arm is tried at least a fixed number of times.
The zooming algorithm proposed by \citet{kleinberg_2019} refines the region for sampling the arms based on the arms similarity.
In recent years, there is a surge of literature dealing with bandit algorithms in countably many, and, the closely related, infinite arm settings. The reader is advised to refer to \cite{KalvitZ20,yinglun2020,bayati_2020,kleinberg_2019,lattimore2020bandit} and references therein for a detailed coverage of different algorithms. 

In all the above discussed work, candidate arm selection is not carried out in a structured fashion and does not make use of the arms context effectively.
%information associated with the arms.
To the best of our knowledge, our proposed method is the first non-random candidate arm selection strategy that makes use of the context information, which is evaluated in the context of sequential query recommendation systems.

%\subsection{Query Recommendations}
\paragraph{\textbf{Query Recommendations:}}
Query recommendation based on historical query logs is extensively studied in the Information Retrieval (IR) and Database communities \citep{baeza2004query,he2009web,dimitriadou2014explore,he2009web,li2019click,dehghani2017learning,rosset2020leading}.
Most of the earlier work on query recommendation made use of simple IR techniques like query similarity and query support \citep{baeza2004query}.
\citet{dimitriadou2014explore} proposed a query recommendation algorithm for interactive data exploration applications.
The proposed algorithm works by building on-the-fly decision tree classifiers from the user feedback obtained before the exploration starts.
Other prominent approaches include building offline probabilistic models using historical session data. % to predict the next query to run.
In \cite{he2009web}, supervised \emph{N-gram} based Markov model is trained on historical session data to predict the sequence of next queries.
The Click Feedback-Aware Network algorithm proposed in \cite{li2019click}, models the sequential queries using deep neural networks.
Given a query, the proposed method predicts a ranked list of queries for recommendations.
The model is trained on positive and negative query instances, constructed from the historical logs.
Similar to \cite{li2019click}, \citet{dehghani2017learning} proposed a sequence-to-sequence Recurrent Neural Network model with query-aware attention mechanism. 
\citet{jiang2018rin} also used sequence-to-sequence recurrent neural networks with attention mechanisms to recommend next query.
The main difference between the above two approaches is: \citet{jiang2018rin} used a `query reformulation inferencer' to obtain homomorphic embedding of the queries.
As in the case of CMAB, recent years evidenced a surge in the query recommendation literature attributed mainly to the success of sequence-to-sequence deep learning models.
The reader is advised to refer to \cite{wu2018query,li2019click,jiang2018rin,dehghani2017learning} and references therein for a detailed coverage of different query recommendation algorithms. 

All the above work considers query recommendation from a pure supervised / weakly supervised learning setting.
The user feedback is not dealt with in an online fashion. Instead, it is used to create supervised training examples.
In our limited knowledge, our work is the first attempt to formulate query recommendation as a pure online learning task using the MAB framework.

\section{Countably Many Armed Contextual Bandits for Query Recommendations}
\label{sec:qr_rmd}
\subsection{Problem Fundamentals}
Many information filtering tasks, e.g., information discovery in the Web \citep{krause2011submodularity} or exploratory data analysis \citep{marcel2011survey}, involve interactive data exploration through user issued queries, often chosen from the recommendation provided by a query recommendation engine.
In such environments, conceptually, a (user) session starts with a user issuing a task or topic specific initial query to the data system.
Then, the system responds with a ranked list of relevant results and a query recommendation to be potentially executed at the next step of the process.
Depending on the usefulness (or utility) of the recommended query, a user may wish to either execute it or ignore it.
For meaningful exploration, the recommended queries are anticipated to be related to the topics (tasks) corresponding to the currently executing query and, thus, expected to be semantically similar to the current one.
Normally, the recommended queries are taken or derived from the historical query logs containing all queries executed in the past \citep{baeza2004query}.
Hence, the potential queries for recommendation are considered as countably many.
We envisage this problem as an optimal arm selection in stochastic multi-armed bandit settings with countably many arms \citep{KalvitZ20} where queries correspond to arms.
%Often, queries are represented using real valued vectors, assumed to be of unit length, that captures the semantic information contained in the query.
%This semantic information serves as the arm contexts, and we make use of these contexts to select the optimal arm.

Let $\sA$ be the set of countably many arms and $a_i \in \sA$ be the currently played arm.
Our goal is to select an arm from the set $\sA^\prime = \sA \setminus \{a_{i}\}$, which results in the maximum cumulative reward.
The reward in our setting is the same as the reward in a typical recommendation problem using MAB. That is, rewards are in the form of clicks for the correctly recommended queries.
If the cardinality of $\sA^\prime$ is significantly less than the time horizon, one could use standard stochastic multi-armed bandit algorithms like Upper Confidence Bound (UCB) \citep{auer2002finite} or $\epsilon$-greedy \citep{lattimore2020bandit}.
However, as pointed out earlier, in query recommendation settings, the number of arms is countably many and any standard multi-armed bandit algorithm has to choose from an extremely large number of arms.
The sheer volume of possible arms makes it impossible to try each of the arms (even once), and consequently the algorithm is enforced to ignore many arms and dedicate the valuable query exploration only to a certain number of arms.

Standard CMAB algorithms randomly select a fixed number of arms, called \emph{candidate set}, from the pool of arms, and adopt standard MAB algorithms over this reduced arms set \citep{wang2008algorithms,KalvitZ20,yinglun2020,bayati_2020}.
The random selection might result in the optimal arms to be ignored in the selection process, thus, hindering the performance of the recommendation engine. In our method, we make use of the shared context information between queries to select the candidate set.
We assume that the queries are represented as $d$-dimensional real-valued vectors encoding the semantic content of the queries \citep{le2014distributed,mikolov2013distributed}.
These vectors can be regarded as the contexts associated with the arms.
Without ambiguity, we use the same notation $a_i$ to represent the context vector of arm $a_i$.
Conventional CMAB algorithms do not make use of context information when generating the candidate set.
%In the case of random arm selection, it is possible that our random selection does not contain any optimal arm whereas in the other approach knowing the exact type of each and every arm deterministically is very difficult.
%In stochastic bandit problems with countably many arms, the learner is restricted to ignore many arms without even trying for once and dedicate the valuable exploration scheme only to a certain number of arms.
%That is, in addition to the exploration-exploitation trade-off typical to the sequential learning algorithms, we need to deal with arm selection trade off also.
%In such settings, we need to find optimal arms among the sampled ones and we need to sample enough reasonably potential arms.
%The key component here is that countably many armed bandit algorithms is early identification of the potential candidate arms.
%If we can explore a near optimal arm as early as possible, our regret will be less.

We base our reasoning on using the context associated with an arm by analyzing a real query log files of an online literature discovery service. We plot the sequences of queries executed in five random user-sessions in \autoref{fig:query_plot}.
Specifically, we plot the 2-dimensional query context vectors (we used PCA to reduce the original dimension $d=128$) for five random user-sessions in five different colours (further details about the queries and sessions is given in Section~\ref{sec:exp} and the supplementary file).
From the plot, one can observe that the queries issued within each session are contextually similar with respect to standard similarity measures like $\cos$, and form relatively small loose clusters.
Based on this observation on query similarity, we propose to make use of the shared context information between the currently executing query and queries in $\sA^\prime$ for candidate selection.
The crux of our approach is the assumption that arms with similar context, with respect to a given arm, have similar \emph{utility}.
We formalize this assumption using the following definition of utility adopted from the information-theoretic interpretation of utility function in multi-agent systems. 
Specifically, one can conceive utility in terms of the preference probabilities for being at different states as given in \citep{Ortega2010}.
%but not exactly the same as it is highly unlikely that a user runs the same search or query non-stop; to the already selected arm.
%Moreover, in information filtering tasks, as shown in \autoref{sec:intro}, the optimal arm is correlated to the user played arm.
%So a more meaningful approach to deal with countably many arms will be to sample only the most relevant arms.

%\sh{
%\begin{definition}[{{\cite[Definition 1]{Ortega2010}}}]
\begin{definition}{\textbf{Goodness of Arm for Selection}}~{\cite[Definition 1]{Ortega2010}}
\label{def:good}
Given the currently playing arm $a_{i}$, let $g\big(\{a_j\},\{a_i\}\big) = p(a_{j} \vert a_{i})$ be the conditional probability of the arm $a_{j}$ to be included in the candidate set by a selection strategy.
Then, there exists a real-valued utility function $\util$ which is: (i) sub-additive\footnote{Theorem:\ref{thm:ortega} in \cite{Ortega2010} holds for both additive and sub-additive utilities}, i.e., $\util\big(g\big(\{a_j, a_k\},\{a_i\}\big)\big) \leq \util\big(g\big(\{a_j\},\{a_i\}\big)\big) + \util\big(g\big(\{a_k\},\{a_i\}\big)\big)$; and (ii) consistent, i.e., $g\big(\{a_j\}, \{a_i\}\big) > g\big(\{a_k\},\{a_i\}\big) \Leftrightarrow \util\big(g\big(\{a_j\}, \{a_i\}\big)\big) > \util\big(g\big(\{a_k\}, \{a_i\}\big)\big)$.
\end{definition}

The conditional probability, $p(a_j \vert a_i)$ can be considered as the normalized preference score for the arm $a_j$ to be played next conditioned on the event that $a_i$ is the currently playing arm.
The utility function $\util$ defined above assigns a scalar value to each possible arm such that arms with higher utility correspond to arms that are more preferred.
However, the challenge in our setting is that the preference scores $p(a_j \vert a_i)$ are not known \emph{a priori}.

%Mathematically, these preferences can be formalized by the concept of utility function that assign a numerical value to each possible events such that the events with higher utility correspond to the events with higher preference score.
%We can consider utility function as a surrogate for the actual rewards, and according to Assumption:~\ref{ass:smooth}, this satisfies the condition that events with higher utility correspond to events with higher preference score.
%divergence function $D$ and $\epsilon \geq 0$, $0 \leq \delta \leq 1$, rewards obtained from $(a_i,a_j) \in \sA^\prime$ is similar, if with probability $1-\delta$,  $D(p(a_j|a_i)||p(a_k|a_i)) \leq \epsilon D(r_j|a_i||r_k|a_i) $ . $D(p(a_j|a_i)||p(a_k|a_i))$ is the divergence between the distribution of the arms $a_j$ and $a_i$ conditioned on the event $a_i$ and $D(r_j|a_i||r_k|a_i)$ is the divergence between the reward distributions of $a_i$ and $a_j$ conditioned on the event $a_i$.
%\end{assumption}
%Our assumption is different from the smoothness assumptions used in the fairness analysis of algorithms \cite{dwork2012fairness,liu2017calibrated}.
%The smoothness assumption in algorithmic fairness links the arms similarity with the prediction similarity of a bandit algorithm.

\subsection{Arm Preference Probability}
Typical multi-armed bandit algorithms work by adaptive hypothesis testing.
From an abstract point, such algorithms randomly pick two arms (assuming only two types of arms: optimal and non-optimal) and run on-the-fly hypothesis testing for a predefined duration to estimate statistical properties of the arms.
As noted earlier, contextually similar arms are preferred and, thus, query similarity is correlated to its utility.
Hence, instead of randomly selecting the candidate set, we argue that the pair of arms is chosen using a joint probability distribution defined over the similarity of the arms. 
%which depends on either the optimality gap of the arms or a predefined sequence value to estimate the arm statistics.
%To cope up with the huge arm space, state-of-the-art algorithms either select a fixed random arms to sample or assume that arms can be deterministically assigned to a fixed number types with any from a one type is an optimal arm.
%But often in practice, both approaches are sub-optimal.

Given two arms $a_i, a_j$ and a similarity function evaluation oracle $s_{j,i} = \simf(a_j,a_i)$, let $S_{\geq}^i(\varepsilon)=\{a_j : \simf(a_j,a_i)\geq \varepsilon\}$ and $S_{<}^i(\varepsilon)=\{a_j : \simf(a_j,a_i) < \varepsilon\}$ be the partition of the arms for a given similarity threshold value $\varepsilon>0$.
Based on this similarity indices, a distribution is induced on the elements of $S_{\geq}^i(\varepsilon)$ ($S_{<}^i(\varepsilon)$) as:

\begin{align*}
s_{j,i}(\varepsilon)=p(a_j | a_j \in S_{\geq}^i(\varepsilon)) &= {s_{j,i} \over \sum_k s_{k,i}\llbracket a_k \in S_{\geq}^i(\varepsilon)\rrbracket} \\
\bar{s}_{j,i}(\varepsilon)= p(a_j | a_j \in S_{<}^i(\varepsilon)) &= {s_{j,i} \over \sum_k s_{k,i}\llbracket a_k \in S_{<}^i(\varepsilon)\rrbracket},
\end{align*}
where $\llbracket \cdot\rrbracket$ is the indicator function. %, that is $\mathcal{I}(E)=1$ if $E$ is true and 0 otherwise. 
Extending over the whole population of arms, we define
\[
\pi_{\star,i}(\varepsilon)=p(S_{\geq}^i(\varepsilon))=
{\sum_j s_{j,i}\llbracket a_j \in S_{\geq}^i(\varepsilon) \rrbracket  \over 
\sum_j s_{j,i} } \; \text{ and } \; \bar{\pi}_{\star,i}(\varepsilon) = 1-\pi_{\star,i}(\varepsilon).
\]

Given the currently playing arm $a_{i}$, we posit that the candidate set of arms for adaptive hypothesis testing is sampled according to the following joint distribution; note, with $(s_{j,i},s_{k,i}) \geq \varepsilon$i, we denote the pairwise comparison: $s_{j,i} \geq \varepsilon \wedge s_{k,i} \geq \varepsilon$.

\begin{equation}
\begin{split}
p(a_j,a_k|a_i,\varepsilon) = & p(a_j,a_k | (s_{j,i},s_{k,i}) \geq \varepsilon \vee (s_{j,i},s_{k,i}) < \varepsilon) \\
            = &\frac{\pi^2_{\star,i}(\varepsilon) s_{j,i}(\varepsilon)s_{k,i}(\varepsilon)+\bar{\pi}^2_{\star,i}(\varepsilon)\bar{s}_{j,i}(\varepsilon) \bar{s}_{k,i}(\varepsilon)}{\pi^2_{\star,i}(\varepsilon)+\bar{\pi}^2_{\star,i}(\varepsilon)}.
\label{eq:joint_samp}
\end{split}
\end{equation}

%\iffalse
%where $\pi_{\star,i}(\varepsilon) = \frac{\sum_j s_{j,i}\llbracket s_{j,i} \geq \varepsilon \rrbracket}{\sum_{j} s_{j,i}}$, $\bar{\pi}_{\star,i}(\varepsilon) = 1-\pi_{\star,i}(\varepsilon)$, $s_{j,i}(\varepsilon) = \frac{s_{j,i}}{\sum_k s_{k,i}\llbracket s_{k,i} \geq \varepsilon \rrbracket}$, $\bar{s}_{j,i}(\varepsilon) = \frac{s_{j,i}}{\sum_k s_{k,i}\llbracket s_{k,i} < \varepsilon\rrbracket}$
\paragraph{\textbf{Derivation:}}
By assumption, the candidate set of arms $(a_j,a_k)$ for adaptive hypothesis testing are selected among those that satisfy 
$(a_m,a_n)\in S_{\geq}^i(\varepsilon) \times S_{\geq}^i(\varepsilon)$ or 
$(a_m,a_n)\in S_{<}^i(\varepsilon) \times S_{<}^i(\varepsilon)$. Therefore, 
the probability of selecting the pair $(a_j,a_k)$ is 
\begin{eqnarray}
p(a_j,a_k|a_i,\varepsilon)&=& 
p(a_j,a_k | {S_{\geq}^i(\varepsilon)}^2 \cup
{S_{<}^i(\varepsilon)}^2) 
= {p(a_j,a_k, {S_{\geq}^i(\varepsilon)}^2 \cup
{S_{<}^i(\varepsilon)}^2 ) \over
p({S_{\geq}^i(\varepsilon)}^2\cup {S_{<}^i(\varepsilon)}^2)} \nonumber \\
&=& {p(a_j,a_k,{S_{\geq}^i(\varepsilon)}^2) +
p(a_j,a_k,{S_{<}^i(\varepsilon)}^2 )  \over
p({S_{\geq}^i(\varepsilon)}^2) + p({S_{<}^i(\varepsilon)}^2) } \nonumber \\
&=& {p({S_{\geq}^i(\varepsilon)}^2) p(a_j,a_k |{S_{\geq}^i(\varepsilon)}^2) +
p({S_{<}^i(\varepsilon)}^2) p(a_j,a_k |{S_{<}^i(\varepsilon)}^2) \over 
p({S_{\geq}^i(\varepsilon)}^2) + p({S_{<}^i(\varepsilon)}^2) } \nonumber \\ 
&=& 
{\pi_{\star,i}(\varepsilon)^2 s_{j,i}(\varepsilon)s_{k,i}(\varepsilon) +
(1-\pi_{\star,i}(\varepsilon))^2 \bar{s}_{j,i}(\varepsilon)\bar{s}_{k,i}(\varepsilon) \over
\pi_{\star,i}(\varepsilon)^2 +(1-\pi_{\star,i}(\varepsilon))^2 },
\nonumber
\end{eqnarray}
where the second line follows from the fact that $S_{\geq}^i(\epsilon)$ and
$S_{<}^i(\epsilon)$ are mutually exclusive events, whereas the last line follows from the independence of arms. 
%\fi

%The detailed derivation of the above quantity is given in the supplementary file.
The rationale behind the above sampling scheme is that, given the currently playing arm $a_i$, we independently draw two arms $a_j$ and $a_k$ and accept/reject them only if pairs are similar/dissimilar to $a_i$ with confidence $\varepsilon$.
To understand the concept, a detailed worked-example is given in the supplementary file.
%.(reject) with the corresponding joint probability, for hypothesis testing, if the similarity values of the pairs of arms are larger (smaller) than a predefined threshold $\varepsilon > 0$.

%\begin{figure}[h!]
%\centering{
%\begin{tabular}{@{}c@{}}
%\includegraphics[scale=0.5]{plot.pdf}
%\end{tabular}
%\caption{Five clusters of 2-dim. query vectors corresponding to sequentially correlated executed queries per user-session. Each cluster corresponds to a different user-session/task.}
%\label{fig:query_plot}
%}
%\end{figure}

\paragraph{\textbf{An Illustrated Example}}
\begin{example}
Assume, we have four arms with contexts $a_0,a_1,a_2,a_3$.
Let us say, user started the information gathering session by playing the arm $a_0$.
The learning algorithm has to choose one of the arm from the remaining three arms (1,2,3) for recommendation.
Assume that the similarity values are  $\simf(a_1,a_0) = 0.4, \simf(_2,a_0) = 0.55, \simf_{a_3,a_0} = 0.6$.
Let us take $\varepsilon = 0.5$, now estimating the different quantities in the equation, we get

\begin{equation}
\begin{split}
\pi_{\star,0}(0.5) & = \frac{0.55 + 0.6}{0.4 + 0.55 + 0.6} = 0.742 \\
\pi_{\star,0}^2 &= 0.550 \\
\bar{\pi}_{\star,0} & = 0.258,\, \quad \bar{\pi}_{\star,0}^2 = 0.067 \\
s_{1,0}(0.5) &= \frac{0.4}{0.55 + 0.6} = 0.348 \\
\bar{s}_{1,0}(0.5) &= \frac{0.4}{0.4} = 1 \\
s_{2,0}(0.5) &= \frac{0.55}{0.55 + 0.6} = 0.478 \\
\bar{s}_{2,0}(0.5) &= \frac{0.55}{0.4} = 1.375 \\
s_{3,0}(0.5) &= \frac{0.6}{0.55 + 0.6} = 0.522 \\
\bar{s}_{3,0}(0.5) &= \frac{0.6}{0.4} = 1.5
\end{split}
\end{equation}
Now, estimating the joint probability for sampling the pairs, we get

\begin{equation}
\begin{split}
p(a_1,a_2|a_0,0.5) &= \frac{0.550 \times 0.348 \times 0.478 + 0.067 \times 1 \times 1.375}{0.550 + 0.067} = 0.297 \\
p(a_1,a_3|a_0,0.5) &= \frac{0.550 \times 0.348 \times 0.522 +  0.067 \times 1 \times 1.5}{0.550 + 0.067} = 0.324 \\
p(a_2,a_3|a_0,0.5) &= \frac{0.550 \times 0.478 \times 0.522 +  0.067 \times 1.375 \times 1.5}{0.550 + 0.067} = 0.445
\end{split}
\end{equation}

Now, let us try the same with $\varepsilon = 0.6$

\begin{equation}
\begin{split}
\pi_{\star,0}(0.6) &= \frac{0.6}{0.4 + 0.55 + 0.6} = 0.387 \\
\pi^2 &= 0.150 \\
\bar{\pi}_{\star,0} &= 0.613, \, \quad \bar{\pi}^2 = 0.376 \\
s_{1,0}(0.6) &= \frac{0.4}{0.6} = 0.667 \\
\bar{s}_{1,0}(0.6) &= \frac{0.4}{0.4+0.55} = 0.421 \\
s_{2,0}(0.6) &= \frac{0.55}{0.6} = 0.917 \\
\bar{s}_{2,0}(0.6) &= \frac{0.55}{0.4+0.55} = 0.579 \\
s_{3,0}(0.6) &= \frac{0.6}{0.6} = 1 \\
\bar{s}_{3,0}(0.6) &= \frac{0.6}{0.4+0.55} = 0.632
\end{split}
\end{equation}

Now, as before, let us estimate the joint probability, we get

\begin{equation}
\begin{split}
p(a_1,a_2|a_0,0.6) &= \frac{0.150 \times 0.667 \times 0.917 + 0.376 \times 0.421 \times 0.579}{0.150 +  0.376} = 0.349 \\
p(a_1,a_3|a_0,0.6) &= \frac{0.150 \times 0.667 \times 1 + 0.376 \times 0.421 \times 0.632}{0.150 +  0.376} = 0.380 \\
p(a_2,a_3|a_0,0.6) &= \frac{0.150 \times 0.917 \times 1 + 0.376 \times 0.579 \times 0.632}{0.150 +  0.376}  = 0.523
\end{split}
\end{equation}
\end{example}

Estimating the joint probability for every pair of arms is time consuming and computationally inefficient.
In Lemma~\ref{lemma:marginal}, we show that arms similar to the currently playing arm are selected with the probability in (\ref{eq:sel_prob}). This will help us to select a candidate set based on marginal probabilities only.

\begin{lemma}
\label{lemma:marginal}
Given the set of arms $\sA^\prime$ and the currently playing arm $a_i$, candidate arms are independently drawn with probability
\begin{equation}
\label{eq:sel_prob}
    P(a_j | a_i,\varepsilon) = \frac{\pi^2_{\star,i}(\varepsilon) s_{j,i}(\varepsilon) + \bar{\pi}^2_{\star,i} s_{j,i}(\varepsilon)}{\pi^2_{\star,i}(\varepsilon)+\bar{\pi}^2_{\star,i}(\varepsilon)}.
\end{equation}
\end{lemma}

%\begin{proof}
%The proof is provided in the supplementary extended version.
%\end{proof}

%\iffalse
\begin{proof}
From our assumption, arm pairs $\{(a_j,a_k)\}$ from $\sA^\prime$ are drawn $\sim P(a_j,a_k|a_i)$. In order to get the marginal probability, we sum over $a_k$,
\begin{equation*}
\begin{split}
P(a_j|a_i)=&\sum_{k \neq j} P(a_j,a_k|a_i) \\
    =&\frac{\pi^2_{\star,i}(\varepsilon) s_{j,i}(\varepsilon) \sum\limits_k s_{k,i}(\varepsilon) + \bar{\pi}^2_{\star,i}(\varepsilon) \bar{s}_{j,i}(\varepsilon) \sum\limits_k \bar{s}_{k,i}(\varepsilon)}{\pi^2_{\star,i}(\varepsilon)+\bar{\pi}^2_{\star,i}(\varepsilon)} \\
    =&\frac{\pi^2_{\star,i}(\varepsilon) s_{j,i}(\varepsilon) + \bar{\pi}^2_{\star,i}(\varepsilon) \bar{s}_{j,i}(\varepsilon)}{\pi^2_{\star,i}(\varepsilon)+\bar{\pi}^2_{\star,i}(\varepsilon)} \\
\end{split}
\end{equation*}
\end{proof}

%The proof for Lemma~\ref{lemma:marginal} is given in the supplementary file. 
A naive approach for generating candidate arms set will be to fix a probability threshold and take all the arms with highest marginal probability according to \eqref{eq:sel_prob}, or randomly sample $k$ arms with the estimated marginal probability.
However, such an approach is not effective in practice, i.e., users do not prefer to issue the most similar queries in data exploration tasks, as demonstrated in our experimental study in Section~\ref{sec:exp}.
Moreover, in the CMAB settings, there can be a huge number of similar arms for a fixed threshold. So randomly picking $k$ number of arms as per the marginal probability might lead the optimal or near-optimal arms to be ignored as discussed in our problem statement.
The number of arms for different threshold values of marginal probability is given in \autoref{tab:eps_stat}. 
To this end, we propose to select a candidate arms set that maximizes the \emph{utility} of the arms as given in Definition~\ref{def:good}.
Furthermore, we elaborate on an easy-to-implement distributed framework for finding a candidate set with varying numbers of candidate arms.

\subsection{Candidate Arm Selection Using Maximum Utility}
To select arms based on \emph{utility}, we need a function that realizes real valued utilities from the probabilistic preference scores obtained using the similarity between shared contextual information content.
According to Theorem \ref{thm:ortega} \citep{Ortega2010}, the \emph{logarithm} function is the only function that can express such a relationship between the preference probability scores and utility function. For completeness, we restate Theorem \ref{thm:ortega} by \citet{Ortega2010}.

\begin{theorem}[{\cite[Theorem 1]{Ortega2010}}]
\label{thm:ortega}
Given the arm set $\sA$ with the probability space defined as in \eqref{eq:sel_prob}, a function~$\util$ is a utility function on the probability space, if and only if for all $a_i,a_j \in \sA$, $\util\big(g\big(\{a_j\},\{a_i\}\big)\big) = c \cdot \log\big(g\big(\{a_j\},\{a_i\}\big)\big)$, where $c > 0$ is an arbitrary constant.
\end{theorem}

It can be easily verified that the $\log$~function satisfies all the properties given in the Definition~\ref{def:good} with our preference probability space defined in \eqref{eq:sel_prob}.
Given $\log$ as the utility function, $c=1$ and currently executing arm $a_i$, a candidate set $\sC$ of cardinality~$k$ realizes the optimal utility if it solves\footnote{For the set $\sC$, $g(\sC,a_i)$ is defined as the joint probability of the arms in $\sC$ conditioned on $a_i$}:

\begin{equation}
\max_{\substack{\sC \subseteq \sA^\prime \\ \lvert\sC\lvert \leq k}} \log\big(g\big(\sC,\{a_i\}\big)\big)
\label{eq:opt_set}
\enspace
\end{equation}

Since $\log$ is a concave function, the objective function in \eqref{eq:opt_set} is a submodular maximization problem with cardinality matroid constraint.
Though submodular maximization is NP-hard in general, a simple greedy heuristic due to \cite{nemhauser1978analysis} guarantees a solution with constant approximation factor equal to $1-\frac{1}{e}$. Our CMAB algorithm with max-utility based candidate arm selection procedure is given in Algorithm~\ref{alg:MCMCB}.
\iffalse
\begin{algorithm}[t]
    \caption{Max-utility based Countably Many-armed Contextual Bandits}
    \label{alg:MCMCB}
    \begin{algorithmic}[1]
        \STATE {\bfseries Input:} $\sA^\prime,a_i$
        \STATE {initialize $\sC$ by executing Algorithm~\ref{alg:fast-sub}}
        \FOR {$t = 1, \ldots, T$}
        \STATE {run Algorithm~\ref{alg:select-strategy} to update $\sC$}
        \STATE {run standard stochastic contextual bandit algorithm on arm set $\sC$ and $a_i$} \label{line:band_alg}
        \ENDFOR {}
    \end{algorithmic}
\end{algorithm}
\fi
The input to the algorithm is $\sA^\prime$ and the currently running query (arm) $a_i$.
The algorithm starts with an initial set of candidate arms with fixed cardinality.
This set is constructed following the greedy heuristic.
At each trial, the algorithm dynamically updates the candidate set based on the current trial count using Algorithm~\ref{alg:select-strategy}.
At the initial stages of trials, we allow the candidate set to grow as the algorithm is allowed to sacrifice short term gain by exploring for arms with even larger reward that will expect to yield a long-term benefit. As the trial progresses we keep the candidate set fixed as the algorithm is allowed to eschew new arm exploration, by exploiting the unimodality of $t^\alpha{\rm e}^{-t}$.
Finally, we run the standard stochastic contextual MAB algorithm using the candidate set and currently executing arm.

\begin{figure}
    \begin{minipage}[t]{0.5\textwidth}
        \vspace{0pt}
        \begin{algorithm2e}[H]
            \SetAlgoLined
            \LinesNumbered
            \caption{Max-utility based Countably Many-armed Contextual Bandits}
            \label{alg:MCMCB}
            \SetKwInOut{Input}{Input}\SetKwInOut{Output}{Output}
            \Input {$\sA^\prime,a_i$}
            Initialize $\sC$ by executing Algorithm~\ref{alg:fast-sub}
            \For {$t = 1, \ldots, T$} {
                run Algorithm~\ref{alg:select-strategy} to update $\sC$\;
                run standard stochastic contextual bandit algorithm on arm set $\sC$ using $a_i$ as context\label{line:band_alg}\;
            }
        \end{algorithm2e}
    \end{minipage}\hfill
    \begin{minipage}[t]{0.49\textwidth}
        \vspace{0pt}
        \begin{algorithm2e}[H]
            \SetAlgoLined
            \LinesNumbered
            \caption{Candidate Selection}
            \label{alg:select-strategy}
            \SetKwInOut{Input}{Input}\SetKwInOut{Output}{Output}
            \Input {$\sA^\prime, \sC, t, \alpha \geq 1$}
            $k_t = t^\alpha{\rm e}^{-t}$ \label{line:kn}\;
            $\sS = \sA^\prime \setminus \sC $ \;
            \While { $\lvert \sC \rvert < k_t$ } {
                $e^\star = \argmax_{e \in \sS} \log\big(g\big(\sC \cup \{e\},\{a_i\}\big)\big)$\;
                $\sC = \sC \cup \{e^\star\}$\;
                $\sS = \sS \setminus \{e^\star\}$\;
            }
        %\STATE {select $\sP$ such that $\argmax_{\substack{\sP \subset \sA\setminus\sC \\ \lvert \sP\rvert \leq \left\lceil K_n \right\rceil}} \util(\sP|)$}  
        \Output {$\sC$}
        \end{algorithm2e}
    \end{minipage}
\end{figure}

\subsection{Fast Submodular Function Evaluation}
The standard greedy algorithms for submodular maximization guarantee a near-optimal candidate selection, without room for further improvement using current computing environments.
However, greedy algorithms do not scale well when applied to massive data and, thus, initialization of $\sC$ will incur considerable computation time.
The greedy algorithms work well for centralized submodular maximization problems; but this requires $O(nk)$ value oracle calls to select $k$ arms from $n$ arms.
The adaptive addition of new arms to the existing $\sC$ (Algorithm~\ref{alg:select-strategy}) can be achieved in linear time, as the number of arms to be added will be relatively very small.
For large data, submodular maximization with cardinality constraint can be solved in nearly linear time using a distributed greedy algorithm \citep{mirzasoleiman2016distributed}.
%We adopt the distributed greedy heuristic algorithm, give in Algorithm \ref{alg:select-strategy}, to select the initial candidate set.
We use a faster version of the submodular maximization algorithm \citep{mirzasoleiman2016distributed}, which can be parallelized. The proposed algorithm adopted to our setting is then given in Algorithm \ref{alg:fast-sub}.
In \autoref{line:log_set1} \& \autoref{line:log_set2} of the Algorithm~\ref{alg:fast-sub}, as mentioned earlier, $g$ is the joint probability of all the arms in the input set conditioned on the currently playing arm. %, as the extended version \autoref{eq:joint_samp}.
Due to the use of a greedy algorithm, we do not explicitly calculate it, but incrementally construct the set.

%\iffalse
\begin{algorithm2e}[t]
    \SetAlgoLined
    \LinesNumbered
    \caption{Distributed Submodular Maximization}
    \label{alg:fast-sub}
    \SetKwInOut{Input}{Input}\SetKwInOut{Output}{Output}
    \Input {$\sA^\prime, k, m, a_i$}
    Partition $\sA^\prime$ into $m$ sets $\sA_1,\sA_2,\cdots \sA_m$ (arbitrarily/at random)\;
    Run standard (lazy) greedy algorithm on each $\sA_i$ to obtain solution with cardinality $k$ that maximizes $\log\big(g\big(\sA_i,\{a_i\}\big)\big)$ to get set $\sP_i$ \label{line:log_set1}\;
    Find $\sP_{max} = \argmax_{\sP_{i|1\dots m}}\log\big(g\big(\sP_i,\{a_i\}\big)\big)$ \label{line:log_set2}\;
    Merge the sets into $\sB = \bigcup\limits_{i=1}^m \sP_i$\;
    Run lazy greedy algorithm on $\sB$ to obtain solution that maximizes $\util$ to get $\sP_{\sB}$\;
    \Output {$\argmax_{e \in \{\sP_{max},\sP_{\sB}\}} \log\big(g\big(e,\{a_i\}\big)\big)$}
\end{algorithm2e}
%\fi
\subsection{Choosing $k$ and Anytime Algorithm}
If the total number of trials $T$ is known beforehand, one can choose a value of $k$ that minimizes the regret associated with the underlying bandit algorithm.
For instance, in the UCB-$\infty$ algorithm in \cite{wang2008algorithms}, for a distribution specific parameter $\beta$, one can choose $k$ to be of the order of $T^\frac{\beta}{2}$ or $T^\frac{\beta}{\beta+2}$ depending on the range of $\beta$\footnote{Better approximation for $k$ can be achieved by adopting the Lambert $W$ function.}.
When the number of trials $T$ is not known in advance, one can adaptively choose $k$ depending on the current trial number.
A MAB algorithm is \emph{anytime}, if the regret bounds on the expected regret holds for all values of $T$ (up to constant factors).
We can make our algorithm \emph{anytime} by employing \emph{anytime} bandit algorithm like UCB-$\infty$ in \autoref{line:band_alg} of Algorithm~\ref{alg:MCMCB}.
For example, in case of using UCB-$\infty$, we can use $k_{t-1} < t^\frac{\beta}{2}$ or $k_{t-1} < t^\frac{\beta}{\beta+1}$ (depending on the value of $\beta$) as shown in \autoref{line:kn} of Algorithm \ref{alg:select-strategy}.
%In practical settings, the value of $K$ should be sufficiently small with respect to $n$ as this guarantees that we have fewer plays using sub-optimal arms and more plays using nearly optimal arms.
%Similarly, the number of $K$ should not be relatively small either, since we desire the best of the $K$ arms to have an expected reward close to the best possible arm.
%Hence, at each time step, we choose $K_n$ of the order of $n^a\mathrm{e}^{-n}$.
%This function monotonically increases until it reaches the maximum and then monotonically decreases. That is, when $n$ is relatively large, the algorithm can sacrifice immediate gain (eschewing an arm with large mean, if necessary) while testing arms and searching for one with an even larger mean that will yield a long-term benefit. When $n$ is small, the algorithm favors the selection of arm that has a immediate large mean.

\section{Experimental Evaluation \& Comparative Assessment}
\label{sec:exp}
We discuss the dataset, baselines, and provide in-depth analysis to verify the performance of the proposed arm selection scheme.
Our source code is available at: ~\url{https://anonymous.4open.science/r/0e5ad6b7-ac02-4577-9212-c9d505d3dbdb/}.
\subsection{Dataset \& Context Vector}
Our experiments are conducted on recent large scale query logs from an online literature discovery service.
The application works by accepting a user query and returning a list of the most relevant research articles to the query.
The logs contained around 4.5 million queries grouped into 547740 user sessions.
For our experimental study, as a pre-processing step, we removed user sessions with less than 4 queries and more than 50 queries.
We also removed sessions containing non-English queries. 
The statistics of the final data used in the experiment are given in \autoref{tab:data_stat}.
%An example of the queries issued during literature discovery service, with session ids anonymized, is given in \autoref{tab:query_type}. Further details about the query log is given in the supplementary file.

%\iffalse
\paragraph{\textbf{Details Of Queries:}}
The log file contains the user issued search queries, in the form of free text, to a popular online literature discover service.
For each user session, the log file contains the unique session id, timestamp and the query text.
The query text covers diverse topics like 'knowledge graph embeddings', 'cryptocurrencies', 'COVID 19' etc.
It also contains non-science topics like politics, behavioural studies etc.

Below, we give three examples of the queries issued during the literature discovery service, with session ids anonymized. % kind of queries issued in different sessions.
In \autoref{tab:query_type}, we list queries issued in three different sessions.
Though the queries can be broadly classified under very specific topics, it covers diverse aspects within the topic.
For example, the queries in the first session are related to detecting coding errors, and within the session the user explores different aspects of the topic by exploring implementation details, details about specific tools and details about specific errors.
%\fi
\begin{table}[h]
\centering
\caption{Example Queries}
\label{tab:query_type}
\begin{tabular}{@{}llp{10cm}@{}}
    \toprule
    Session ID   &     Timestamp              &        Query Text  \\
    \midrule
XXXX495      &  2020-X-X X:04:03  & protocol state fuzzing of tls implementations \\
XXXX495      &  2020-X-X X:08:45  & aflnet a greybox fuzzer for network protocols \\
XXXX495      &  2020-X-X X:33:04  & protocol learning fuzzing \\
XXXX495      &  2020-X-X X:42:33  & improving grey box fuzzing by modeling program behavior \\
XXXX495      &  2020-X-X X:03:22  & poster fuzzing iot firmware via multi stage message generation \\
XXXX495      &  2020-X-X X:05:34  & fuzzguard filtering out unreachable inputs in directed grey box fuzzing through deep learning \\
XXXX495      &  2020-X-X X:13:17  & a functional method for assessing protocol implementation security \\
XXXX495      &  2020-X-X X:19:17  & not all bytes are equal neural byte sieve for fuzzing \\
XXXX229	     &  2020-X-X X:17:21  & 	modeling relational data with graph convolutional networks  \\
XXXX229	     &  2020-X-X X:19:00  & 	representing text for joint embedding of text and knowledge bases  \\
XXXX229	     &  2020-X-X X:19:25  & 	convolutional 2d knowledge graph embeddings  \\
XXXX229	     &  2020-X-X X:21:42  & 	deeppath a reinforcement learning method for knowledge graph reasoning  \\
XXXX078	     &  2020-X-X X:47:55  & 	bitcoin transfer system  \\
XXXX078	     &  2020-X-X X:49:38  & 	cryptocurrency transfer system  \\
XXXX078	     &  2020-X-X X:49:55  & 	cryptocurrency transaction system  \\
XXXX078	     &  2020-X-X X:06:49  & 	cryptocurrency transaction analysis from a network perspective  \\
XXXX078	     &  2020-X-X X:20:19  & 	cryptocurrency trackability  \\
    \bottomrule
\end{tabular}
\end{table}
%\fi

For the query context vector, we used a transformer-based deep neural network \citep{vaswani2017attention} to extract the contexts from the queries.
We train a bidirectional BERT model \citep{devlin2019bert} on queries to get a word level vector embeddings.
We feed these embeddings to a sentence embedding algorithm \citep{reimers2019sentence} to get the final context vector representation for each query in the log.
We set the dimension of the context vectors $d=128$.
\begin{table}[h]
    \begin{minipage}{0.5\textwidth}
        \centering
        \caption{\# arms for different $p(a_j|a_i)$}
        \label{tab:eps_stat}
        \begin{tabular}{@{}cc@{}}
            \toprule
            %$p(a_j|a_i,0.5)\varepsilon$ & COUNT \\
            $p(a_j|a_i,0.5)$ & COUNT \\
            \midrule
            0.25 & 1,116,801 \\
            0.40 & 949,365 \\
            0.50 & 523,750 \\
            0.60 & 161,529 \\
            \bottomrule
        \end{tabular}
    \end{minipage}\hfill
    \begin{minipage}{0.5\textwidth}
        \centering
        \caption{Dataset Statistics}
        \label{tab:data_stat}
        \begin{tabular}{@{}ll@{}}
            \toprule
            Description & COUNT \\
            \midrule
            \# of queries (original) &  4,687,947\\
            \# queries (pre-processed) & 1,120,461\\
            \# user sessions (\texttt{"})& 159,237 \\
            %\# user sessions (-{\textquotedbl}-)& 159,237 \\
            \# avg queries/session (\texttt{"}) & 7 \\
            %\# avg queries/session (-{\textquotedbl}-) & 7 \\
            \bottomrule
        \end{tabular}
    \end{minipage}
\end{table}

\subsection{Models Under Comparison}
We compare our max-utility selection strategy against a variant of the zooming algorithm \citep{kleinberg_2019} and the random selection strategy that is used in standard CMAB algorithms \citep{wang2008algorithms,KalvitZ20,yinglun2020,bayati_2020}.
In random selection, we randomly select $k$ arms and pass to the downstream MAB algorithm whereas in max-utility we follow the Algorithm~\ref{alg:MCMCB}.
In case of random and max-utility schemes, we experimented with four different contextual bandit algorithms (LinUCB, LinThompSamp, Random and Similar) as downstream MAB algorithms and for the zooming variant, we used LinUCB as the downstream MAB algorithm.

\smallskip\noindent\textbf{Linear UCB (LinUCB)} was proposed in the context of news recommendation \citep{li2010contextual} where the algorithm sequentially selects news articles based on contextual information about the users and articles, while simultaneously adapting to the user-click feedback. LinUCB models the reward as a linear function of the context vector.

\smallskip\noindent\textbf{Linear Thomson Sampling(LinThompSamp)}was proposed as an extension to the Thomson sampling scheme to the stochastic contextual bandit settings \citep{agrawal2013thompson}. The rewards are assumed to be a linear function of the context vectors. The algorithm itself is based on Bayesian ideas and assumes that the reward likelihood and mean parameter follow the Gaussian distribution.

\smallskip\noindent\textbf{Most Similar} strategy recommends one query from the top-five queries with highest conditional marginal probability with respect to the currently executing query from the candidate set. % and recommends that as the next query to execute.
In \autoref{fig:query_plot}, we plot the context vectors associated with queries issued in five different user sessions in five different colors.
Each data point represents a query.
The queries in each session are contextually similar and form relatively loose clusters.
Looking at the figure, one might expect that recommending the most similar query might be a useful strategy. We used this baseline strategy to demonstrate that the naive strategy of picking the queries with high marginal probability or similarity with respect to the currently executing query underperform compared to other algorithms.

\smallskip\noindent\textbf{Random} strategy randomly selects a query for recommendation from the candidate set.

\smallskip\noindent\textbf{Zooming LinUCB} is based on the zooming algorithm proposed by \cite{kleinberg_2019}. It combines the upper confidence bound technique with an adaptive refinement step that selects a candidate set region.
For each currently executing query, we select candidate arms that is equal or higher than the similarity threshold and run LinUCB on this arm set.

\subsection{Results \& Analysis}
\begin{figure*}[t!]
    \begin{minipage}{0.49\textwidth}
        \centering
        \begin{tabular}{@{}c@{}}
            \includegraphics[width=\textwidth]{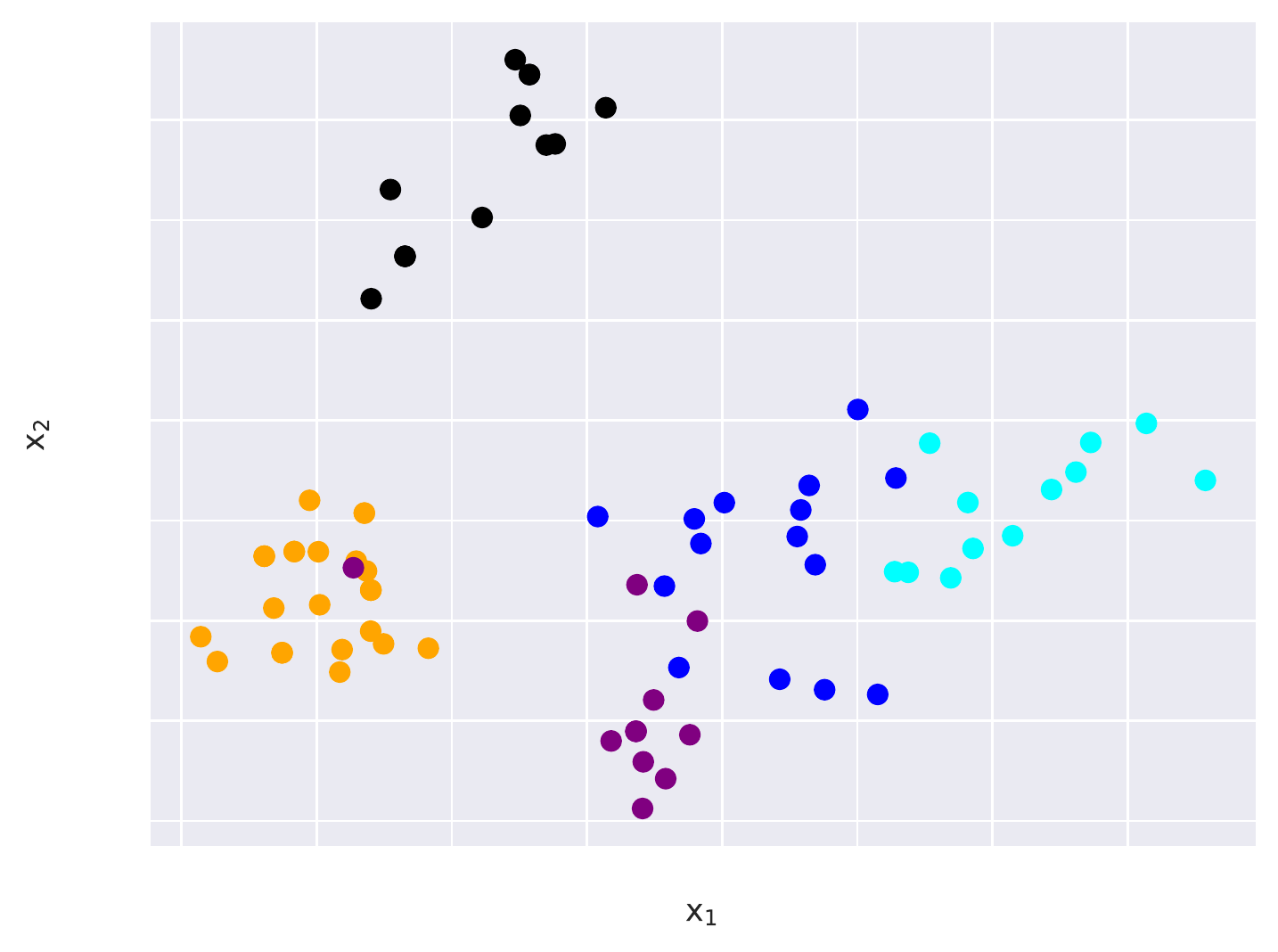}
        \end{tabular}
        \caption{Similarity of the queries in five user sessions in five different colors.}
        \label{fig:query_plot}
    \end{minipage}\hfill
    \begin{minipage}{0.5\textwidth}
        \centering
        \begin{tabular}{@{}c@{}}
            \includegraphics[trim={0 0 0 3em}, clip, width=1.08\textwidth]{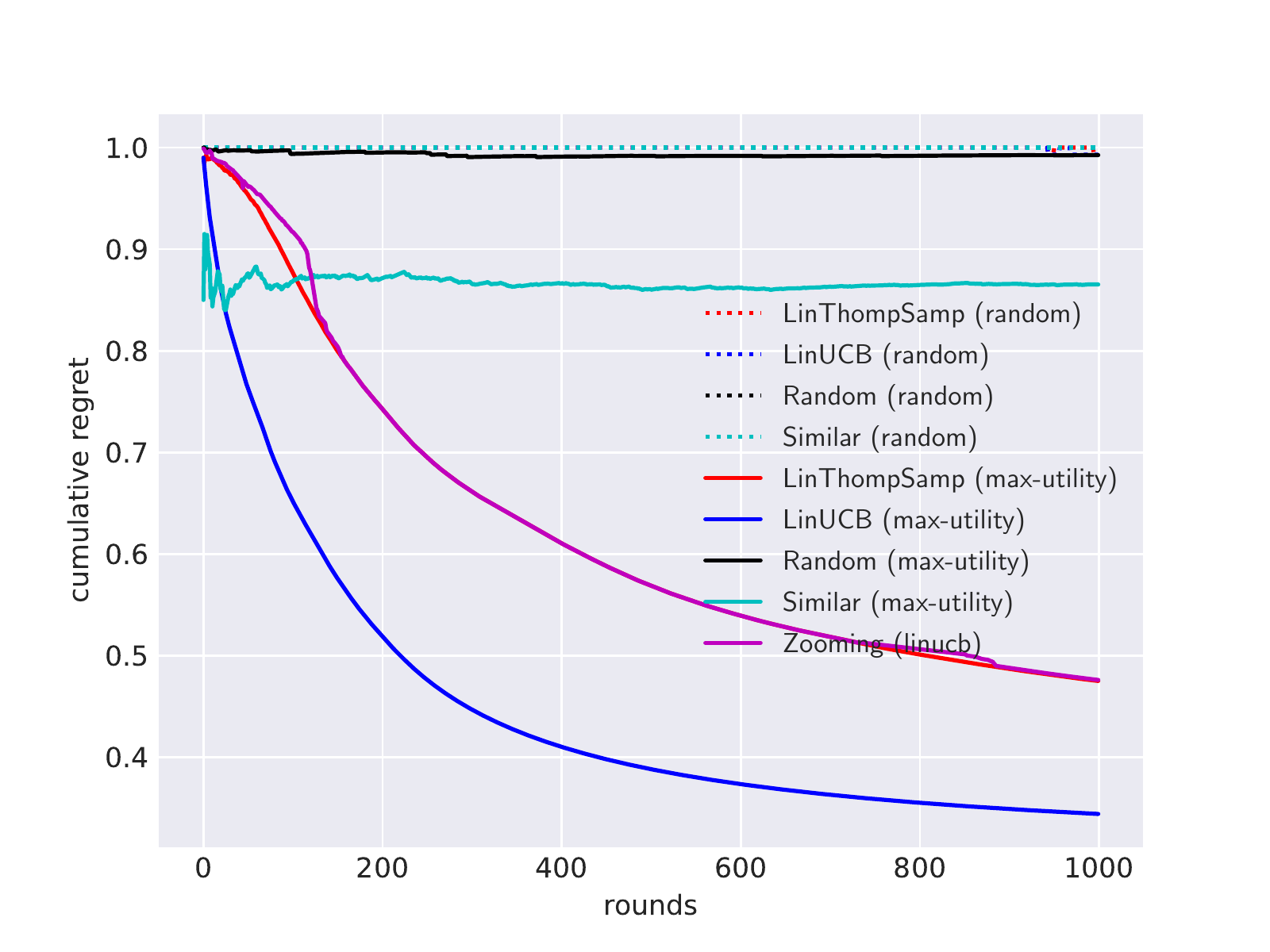}
        \end{tabular}
        \caption{Per-round regret for the different algorithms.}
        \label{fig:round_regret}
    \end{minipage}
\end{figure*}

We used the cumulative regret to compare the performance of the algorithms.
Formally, we compare the quantity \[ R(T) = T - \sum_{t=1}^T r(a^t), \] where $r(a^t)$ is the reward obtained from the arm played at the $t^{th}$ round.
In the experiments, we used the first query in each user session as the initial query and $\sA^\prime$ is given as the input to the candidate set selection algorithm.
We used cosine as the similarity function between arms.
If the recommended query is the one of the next queries executed in the corresponding user session, the algorithm is rewarded with $r(a) = 1$, otherwise, $r(a) = 0$.
The per-round regret of different MAB algorithms against the random and max-utility selection strategies is shown in \autoref{fig:round_regret}.
At round $t$, the per-round regret is defined as $\frac{R(t)}{t}$.
As it can be seen from the plot, random selection strategy results in optimal or near-optimal arms to be ignored whereas max-utility based candidate selection always includes optimal or near-optimal arms in the candidate set.
Though the zooming algorithm with LinUCB performed as well as LinThompSamp with max-utility, it resulted in higher regret compared to LinUCB with max-utility.
Our results are also inline with the regret guarantees proved for LinUCB and LinThompSamp.
In the case of LinUCB, the per-round regret is of the order of $\frac{d}{\sqrt{T}}$ whereas in the case of LinThompSamp it is of the order of $\frac{d^2}{\sqrt{T}}$.
Thus LinUCB resulted in lower regret compared to LinThompSamp with max-utility. % based arm selection strategy.
Even with superior regret guarantees of the LinUCB algorithm, the zooming algorithm failed to achieve the regret of LinUCB with max-utility. 
Looking at the plot in \autoref{fig:query_plot}, one might be tempted to use simple strategies like recommending similar queries (our Similar strategy), but from the performance comparison in \autoref{fig:round_regret} it is very evident that even with the proper candidate selection strategy, it performs very badly in the longer runs.

\begin{figure*}[t!]
    \begin{minipage}{0.5\textwidth}
        \centering
        \begin{tabular}{@{}c@{}}
            \includegraphics[width=\textwidth]{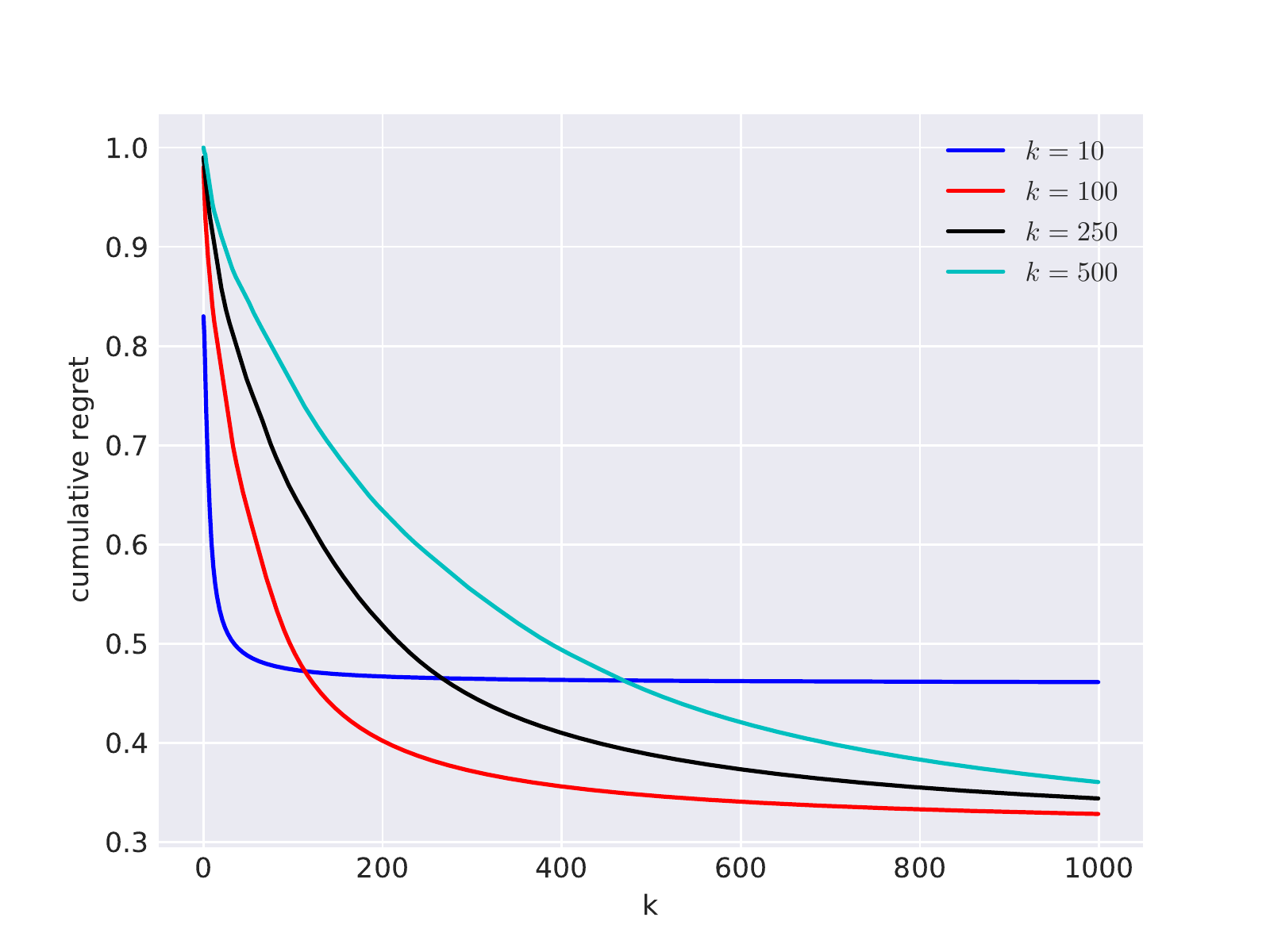}   
        \end{tabular}
        \caption{Per-round regret for different $k$.}
        \label{fig:arm_regret}
    \end{minipage}\hfill
    \begin{minipage}{0.5\textwidth}
        \centering
        \begin{tabular}{@{}c@{}}
            \includegraphics[width=\textwidth]{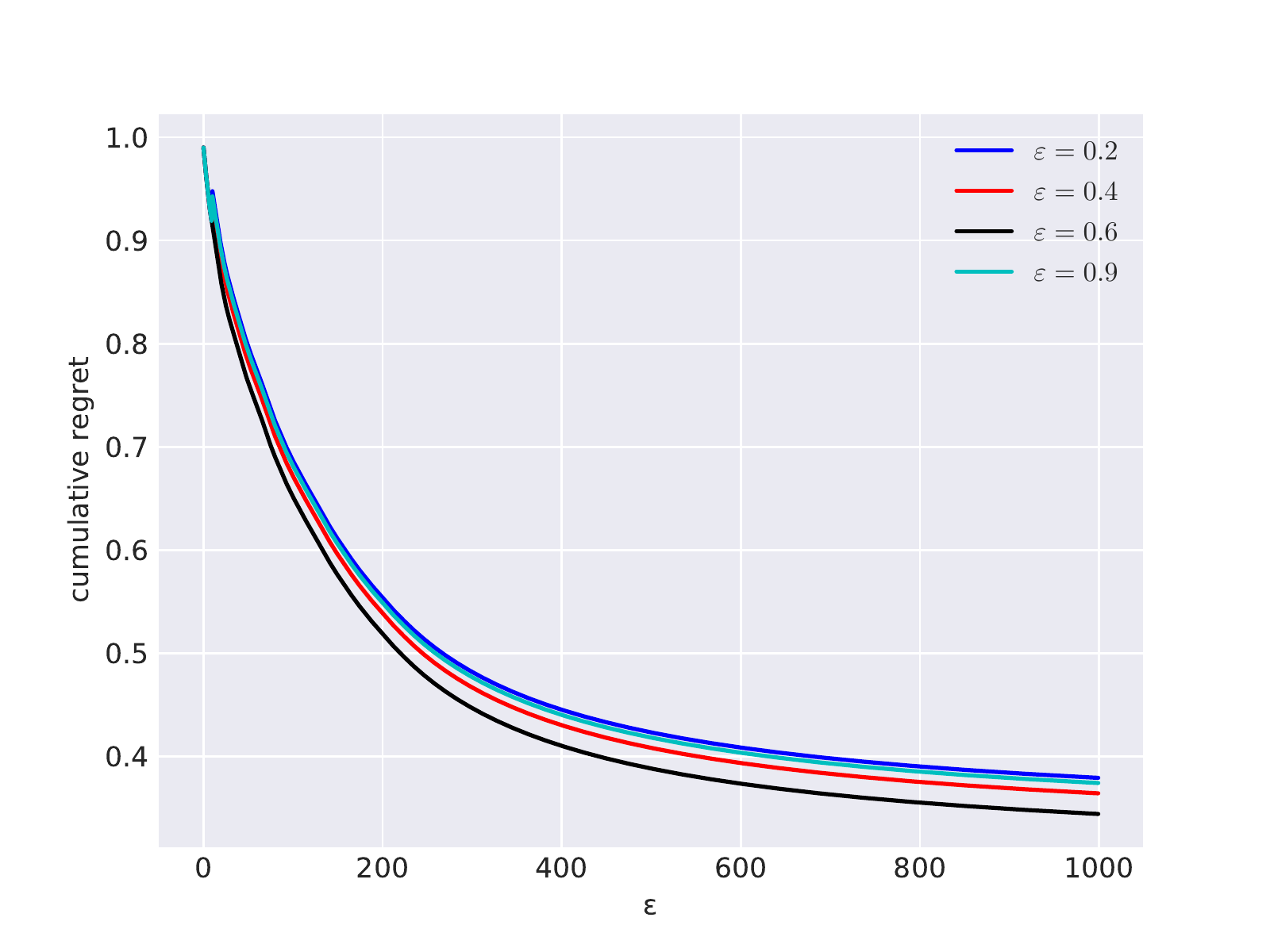}
        \end{tabular}
        \caption{Per-round regret for different $\varepsilon$.}
        \label{fig:eps_regret}
    \end{minipage}
\end{figure*}
Another interesting observation is that with the max-utility arm selection strategy, even the trivial random recommendation strategy performs better than random recommendation with random arm selection strategy.
Moreover, random recommendation with max-utility arm selection strategy results in lower regret than recommending similar queries from a randomly picked candidate set.
\paragraph{\textbf{Hyperparameter Analysis.}}
We further investigated the effect of the size of the candidate set and the similarity threshold ($\varepsilon$) on the algorithm performance.
For this set of experiments, we used LinUCB as the downstream MAB algorithm.
Our results are shown in \autoref{fig:arm_regret} \& \autoref{fig:eps_regret}.
In \autoref{fig:arm_regret}, we varied the value of $k$ from 10 to 500 and compared the performance as a function of~$k$. Here, we used $\varepsilon = 0.5$.
Keeping the cardinality of the candidate set to small or high does not give any performance advantage. Though the average number of queries per session in our dataset is $\sim$7, setting $k=10$ performed very poorly. So it is very important to choose the correct size for the candidate set. Precisely, $k$ should be chosen large enough such that the candidate set contains all possible, relevant and diverse queries with respect to the currently running query.
We also analyzed the performance of the max-utility candidate selection strategy for different values of $\varepsilon$.
Probability preference score for an arm is determined by the value of $\varepsilon$, thus, it is an hyperparameter of the selection strategy.
The performance of the LinUCB algorithm for different values of $\varepsilon$ is plotted in \autoref{fig:eps_regret}. Here, we used $k=250$.
For small and very large values of $\varepsilon$, the per-round regret is slightly higher than mid-range (0.4 - 0.6) values of $\varepsilon$.
By keeping $\varepsilon$ to small and large values, we make the preference probability score between the currently running query and the remaining queries to be high and low respectively.
As a result, we notice the same trend as in \autoref{fig:arm_regret}.
When $\varepsilon$ is small many irrelevant queries will have high preference probability scores.
Similarly, when $\varepsilon$ is large, many diverse but relevant queries will have low preference probability scores.

\section{Conclusions}
\label{sec:conc}
We modelled the query recommendation problem in closed loop interactive learning settings like online information gathering using a MAB framework with countably many arms.
The standard way to solve MAB problems with countably many arms is to select a small set of candidate arms and then apply standard MAB algorithms on this candidate set downstream.
We showed that such a selection strategy often results in higher cumulative regret and proposed a selection strategy based on the maximum utility of the arms.
Our experimental results using a real online literature gathering service log file demonstrated that the proposed arm selection strategy significantly improves the cumulative regret compared to zooming algorithm and the commonly used random selection strategy for a variety of contextual multi-armed bandit algorithms.

%\subsection{Subsection Title}
%A figure in Fig.~\ref{fig:spiral}. Please use high quality graphics for your camera-ready submission -- if you can use a vector graphics format such as \texttt{.eps} or \texttt{.pdf}.
%\begin{figure}[htp]
%\begin{center}
%\includegraphics[width=0.5\textwidth]{spiral.eps}
%\caption{A spiral.}\label{fig:spiral}
%\end{center}
%\end{figure}

%An example of citation~\cite{DBLP:conf/acml/2009}.
%
%\acks{Acknowledgements should go at the end, before appendices and references.}

%\bibliographystyle{abbrvnat}
\bibliography{review}

%\appendix

%\section{First Appendix}\label{apd:first}

%This is the first appendix.

%\section{Second Appendix}\label{apd:second}

%This is the second appendix.

\end{document}